\documentclass[a4paper]{llncs}

\newfont{\titelfont}{cmr10 scaled 1728}
\newfont{\titelbffont}{cmbx10 scaled 2074}
\newfont{\titelbigfont}{cmr10 scaled 2488}
\usepackage[left=1.18 in, right=1.18 in, top=1.18 in, bottom=30mm,
includefoot]{geometry}
\usepackage[ruled, linesnumbered]{algorithm2e}
\usepackage{wrapfig}
\usepackage{amsfonts}
\usepackage{amssymb}
\usepackage{amsmath}
\usepackage{graphicx}
\usepackage{psfrag}
\usepackage{epsfig}
\usepackage{subfigure}
\usepackage{paralist}
\usepackage{mdwlist}
\usepackage{fancyhdr}

\usepackage[english]{babel}

\usepackage[numbers]{natbib}

\fancypagestyle{fancyFirstPage}{%
\headheight 30pt
\footskip 15pt
\fancyhf{} 
\lhead{}
\chead{
  \textrm{
International Journal of Web \& Semantic Technology (IJWest)
Vol.1, Num.4, October 2010}
  }
\rhead{}

\lfoot{
 \textrm{DOI : 10.5121/ijwest.2010.1403}
}
\cfoot{}
\rfoot{\textrm{\thepage}}

}

\newenvironment{parawithheader}[1]{
\begin{trivlist}
\item[\hskip \labelsep {\bfseries [#1 :]}
\rmfamily ]}{\end{trivlist}}

\def\squareforqed{\hbox{\rlap{$\sqcap$}$\sqcup$}}
\def\qed{\ifmmode\squareforqed\else{\unskip\nobreak\hfil
\penalty50\hskip1em\null\nobreak\hfil\squareforqed
\parfillskip=0pt\finalhyphendemerits=0\endgraf}\fi}

\title{
Detecting Ontological Conflicts in Protocols between 
\\Semantic Web Services
}

\author{Priyankar Ghosh and Pallab Dasgupta}
\institute{
Department of Computer Science and Engineering,\\
Indian Institute of Technology Kharagpur, India\\
\email{\{priyankar, pallab\}@cse.iitkgp.ernet.in}
}

\date{}

\begin{document}

\maketitle
\thispagestyle{fancyFirstPage}
\begin{abstract}
The task of verifying the compatibility between interacting web services has
traditionally been limited to checking the compatibility of the interaction
protocol in terms of message sequences and the type of data being exchanged.
Since web services are developed largely in an uncoordinated way, different
services often use independently developed ontologies for the same domain
instead of adhering to a single ontology as standard. In this work we
investigate the approaches that can be taken by the server to verify the
possibility to reach a state with semantically inconsistent results during
the execution of a protocol with a client, if the client ontology is published.
Often database is used to store the actual data along with the ontologies
instead of storing the actual data as a part of the ontology description. It is
important to observe that at the current state of the database the semantic
conflict state may not be reached even if the verification done by the server
indicates the possibility of reaching a conflict state. A relational algebra
based decision procedure is also developed to incorporate the current state of
the client and the server databases in the overall verification procedure.
\end{abstract}

\section{Introduction}\label{IntroSec}
Ontology is regarded as a formal specification of a (usually hierarchical)
set of concepts and the relations between them. The need for developing
intelligent web services that can automatically interact with other web
services has been one of the primary forces behind recent research towards
standardization of ontologies of specific domains of interest
\cite{DBLP:conf/IEEEcit/GuoCL05,DBLP:conf/sac/NoiaSDM03,OWL,WSML,DublinCore}.
For example, if several online book stores follow the same ontology for the {\em
book} domain, then it facilitates an intelligent web service to automatically
search these book stores to find books in a particular category.

In the context of next generation of web, it is envisaged that intelligent
agents will find, combine, and act upon information on the web, thereby
perform the routine day-to-day jobs independently. The protocols that will be
used by such intelligent agents to communicate with the semantic web services,
will play an extremely important role towards materializing the next generation
of web. The protocol may contain branches which are decisions made on the basis
of the previous information exchange. Along with defining the information
exchange between the client and server in the form of a set  query-answer,
independent actions will be described as a part of the protocol. The action may
be automatically executed or may need manual intervention for completion, but
the information required to initiate the action is provided by answer of the
previous queries. We present an example of such protocol in
Section~\ref{ProtoConSec}.

When two communicating web services use ontologies, with respect to semantic
conflict the following scenarios are possible.

\begin{description}
 \item[Scenario-1 :] 
If the web services choose to use the same ontology, there will be no semantic
conflict. In this paper we observe that the requirement that the ontologies used
by communicating web services must match is a very strong requirement which is
often not needed in practice. 

 \item[Scenario-2 :]
If two communicating web services use different ontologies, then they may
potentially reach a state where there is a semantic conflict/mismatch arising
out of the differences between their ontologies. For example, suppose the
ontologies of web service $A$ and web service $B$ recognize the class  $vehicle$
and its sub-classes, namely, \emph{car}, \emph{truck} and \emph{bike}. The
ontology of $A$ defines \emph{color} as an attribute of class 
\emph{vehicle}, where as the ontology of $B$ defines \emph{color} as an 
attribute of the sub-classes \emph{car} and \emph{bike} only. Now suppose
$A$ wants to follow the following protocol with $B$:
\begin{enumerate}
\item[Step-1:] Ask $B$ for the registration number of a vehicle which
is owned by a given person.
\item[Step-2:] If $B$ finds the registration number, then ask $B$ for
the color of the vehicle.
\end{enumerate}
Several executions of this protocol are possible for different valuations of
the data exchanged by the protocol. Semantic conflicts arising out of the
differences in ontologies may occur in some of these cases, but not always.
For example:
\begin{itemize}

\item If $B$ does not find the registration number, then Step-2 is not
        executed and there is no semantic conflict.

\item If $B$ finds the registration number and the vehicle happens to
        be a truck, then Step-2 of the protocol will lead to a semantic
        conflict, since in $B$'s ontology, the {\em color} attribute is
        not defined for trucks.

\item If $B$ finds the registration number and the vehicle happens to
        be a car or a bike, then Step-2 will not lead to a semantic
        conflict, since in $B$'s ontology, the {\em color} attribute is
        defined for cars and bikes.

\end{itemize}
If the ontology of $A$ and the protocol is made available to $B$, then $B$
can formally verify whether any execution of the protocol may lead to a
semantic conflict and warn $A$ accordingly before the actual execution of
the protocol begins. 

There has been considerable research in the recent past on matching ontologies
and finding out semantic conflicts/mismatches among two ontologies
\cite{Visser_analysis,DBLP:journals/jods/CastanoFM06,
DBLP:journals/kbs/HameedSP02}.  In many cases, two web services may have
conflicting ontologies, but the protocol between them may avoid the conflict
scenarios. Consider the scenario where the direction of query-answer is
reversed, that is the same sequence of queries are made by $A$ and answered by
$B$. Also $A$ makes the query about the color of vehicle only if the vehicle is
not a truck. In this case the conflict will not be sensitized by the protocol.
In other words, two agents may not agree on all concepts in their universe, but
may still be able to support certain protocols as long as they avoid the
contentious issues -- a fact which is often ignored in world politics! Therefore
an approach which rules out communication between two services on the grounds
that their ontologies do not match is too conservative in practice. Since the
standardization of ontologies and their acceptance in industrial practice seems
to be a distant possibility, we believe that the verification problem presented
in this paper and its solution is very relevant at present.

 \item[Scenario-3 : ] 
The ontologies can be visualized as a combination of meta-data and a set of
instances. Classes, relations and data-types form the meta-data part of the
ontology, whereas the individuals and the valuations of the attributes are the 
actual data. It is often the case that the actual data is stored in a database,
and ontologies are used as a wrapper on top of the databases. Therefore the
state of the database has to be incorporated, while the server checks whether
the protocol can possibly reach conflict state. Since the protocol between the
client and the server typically have branches and the decision for making the
next query is dependent on the answer of the current query, the conflict that is
present at the ontology level may not be sensitized due to the the answers
generated from the back-end database. We present a relation algebra based
decision procedure to check whether the conflict, that are present in the
ontology level, are actually present with respect to the current state of the
back-end database.

 \item[Scenario-4 : ]
It is important to observe that the protocol has different runs depending on
the instantiation of the variables that are used in the protocol. Since the
conflict may not be sensitized in a particular run of the protocol, the server
may choose to start the protocol and check the possibility to get into a
conflict after every information exchange. Depending on how the conversation
progresses the server may either continue to run protocol, or may terminate the
conversation when it finds that the conflict is inevitable.

\end{description}
A preliminary version of this work is published in~\cite{priyankar-west10}. In
that version we presented the verification algorithm for Scenario-2. In this
paper we include the algorithms for Scenario-3, i.e. the verification of the
spuriousness of an ontological conflict with respect to the current state of the
back-end database. We also show that the same algorithm can be used by the
server for Scenario-4. The paper is organized as follows. The syntax for
describing a protocol is described in Section~\ref{ProtoConSec}. In
Section~\ref{GraphModelSec} we present a graph based model for representing the
ontologies. The proposed formal method for detecting semantic conflicts at the
ontology level is presented in Section~\ref{OntoMethodSec}. The notion of
ontology with database and query answering with the back-end database and the
algorithm to verify the conflicts at the ontology level in the presence of the
database are presented in Section~\ref{OntoDatabaseSec}. Related works are
briefly discussed in Section~\ref{RelatedWorksSec}. Finally we present the
conclusion in Section~\ref{ConclusionSec}.

\section{Protocol and Conflict}\label{ProtoConSec}
In this section we present a formalism similar to SQL for the specification of
the protocol. It may be noted that other formalisms can also be used to specify
a protocol as long as the formalism has expressive power similar to the
formalism used in this paper. We present two example protocols and also
describe the notion of the conflict that we have addressed in this paper.

\subsection{Formal Description of the Protocol}
Typically, a protocol consists of a sequence of queries and answers. The query
specifies a set of variables through \emph{``Get''} keyword and specifies a set
of classes using \emph{``from''} keyword. The valuations corresponding to the
variable set are generated from those classes. Also an optional \emph{``where''}
keyword is used to specify the conditions on the variables. The answer of a
query is a tuple of valuations corresponding to the variable set specified in
the query. The branching is specified using \emph{``if-else''} statements. 

\subsection{Example of Protocol}
\begin{figure}[htb]
\psfrag{C}{Client}
\psfrag{S}{Server}
\psfrag{q11}{$Get (title:\; t1, author:\; a, date:\; d1)$}
\psfrag{q12}{$from$ $Manual$}
\psfrag{q13}{where $t1 = `ManualName'$}
\psfrag{a1}{$\langle t1, a, d1 \rangle$}
\psfrag{q21}{$Get (title:\; t2, author:\; a)$}
\psfrag{q22}{$from$ $Book$}
\psfrag{q23}{}
\psfrag{a2}{$\langle t2, a \rangle$}
\psfrag{q31}{$if$ $(t_2 != null)$}
\psfrag{q32}{$Get (title:\; t3, author:\; a, date:\;d2)$}
\psfrag{q33}{$from$ $Book.Proceedings$}
\psfrag{a3}{$\langle t3, a, d2 \rangle$}
\centering
\includegraphics[scale=0.48]{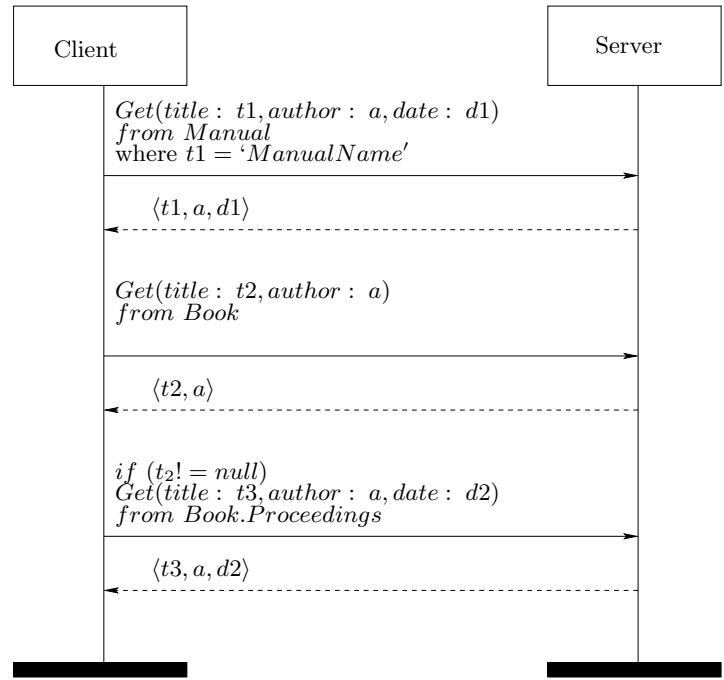}
\caption{Protocol on Publication Domain} \label{Proto1}
\end{figure}
\begin{figure}[htb]
\psfrag{C}{Client}
\psfrag{S}{Server}
\psfrag{q11}{$Get (Brand:\; b1, ItemsSold:\; c1, Year:\; y1)$}
\psfrag{q12}{$from$ $SaleStats$}
\psfrag{q13}{$where$ $(c1 > 10000)(y1 = 2009)$}
\psfrag{a1}{$\langle b1, c1, y1 \rangle$}
\psfrag{q21}{$Get (Brand:\; b1, Model:\; mod, Date:\; d1, Color:\; col)$}
\psfrag{q22}{$from$ $Vehicle.Truck$}
\psfrag{q23}{$where$ $(d1.year > 2000)(col = `Red')$}
\psfrag{a2}{$\langle b1, mod, d1, col \rangle$}
\centering
\includegraphics[scale=0.53]{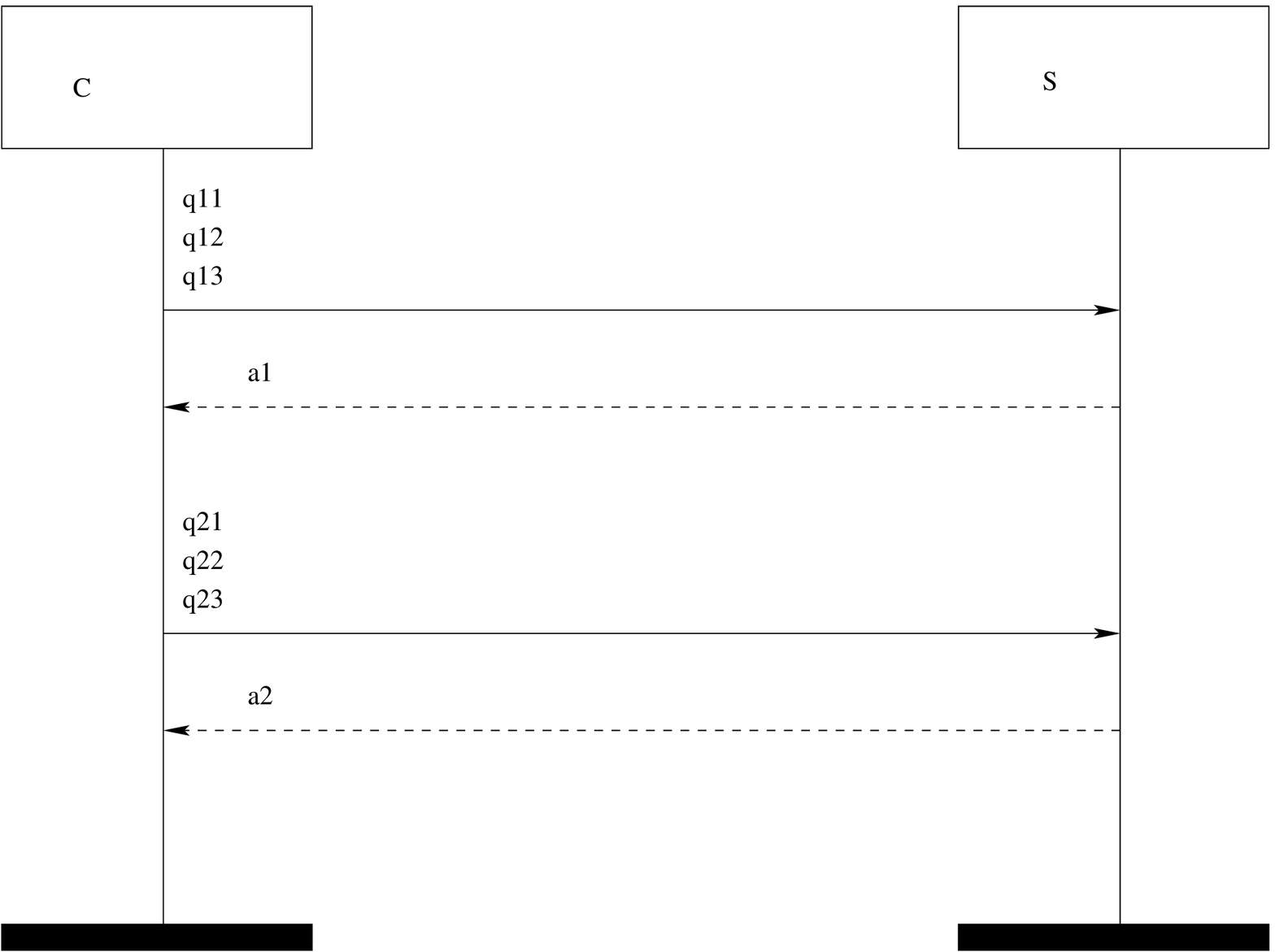}
\caption{Protocol on Automobile Domain} \label{Proto2}
\end{figure}

\begin{parawithheader}{Protocol - 1}
Consider the protocol shown in Figure~\ref{Proto1}. The protocol depicts a
conversation between a client and a server over the publication domain. The
query of the client is about the author of some specific manual. Then the client
makes a query to retrieve a book by the author of that manual. According to
the ontology of the client, \emph{`Proceedings'} is a subclass of \emph{`Book'}
and the client makes the next query to retrieve the proceedings by the same
author. If the server does not recognize \emph{`Proceedings'} as a sub class of
\emph{`Book'}, the query can not be answered by the server due to the mismatch
in the ontologies.
\end{parawithheader}

\begin{parawithheader}{Protocol - 2}
In Figure~\ref{Proto2} we present another protocol that exchanges information
about the automobile domain. The client makes a query to retrieve a brand
which has sold more than a specific number of vehicles in a particular year.
Then next query is made in the context of the previous query to check whether
that brand manufacture \emph{`Red Trucks'}. According to the ontology of
the client the color is a property of the vehicle class and therefore all
subclasses of vehicle class will have the color attribute. However if the server
recognizes \emph{`color'} as an attribute of some of the sub-classes(suppose
\emph{`car'} and \emph{`two-wheeler'}) instead of as an attribute of the class
\emph{`Vehicle'} itself, the query can not be answered by the server due to the
mismatch in the ontology.
\end{parawithheader}

\begin{parawithheader}{Protocol - 3}
In this example we present a protocol of an \emph{intelligent agent}. Consider
the semantic web service for an online store. The online store can queried to
retrieve the relevant information about the available items. Also consider a
multi-cuisine restaurant which is a client of that store. Whenever the stock of
some item, say $i_1$, falls below some level, the intelligent agent that works
on behalf of the restaurant, searches the availability of $i_1$ by querying the
online store. Suppose $i_1$ comes in two qualities, $q_1$ and $q_2$. The
protocol, that is used by that agent to find and buy the item under
consideration, is presented below using a format similar to pseudo code. Here
the \emph{buy} action is carried out by the agent  automatically, if the
precondition is satisfied.

\begin{footnotesize}
\begin{tabbing}
aaaa \= aaaa \= aaaa \= aaaa \= aaaa \= aaaa \kill
\rule{\linewidth}{0.02cm}
\\
\bf{Protocol for Buying an Item}\\
\rule[2mm]{\linewidth}{0.02cm}
\\
Get the availability $i_1$ of quality $q_1$;\\
If ($i_1$ of quality $q_1$ is available) \\
\>Get the price of $i_1$ of quality $q_1$;\\
\>If (the price is less than $C_1$) \\
\>\>Buy $i_1$ of quantity $Q_1$;\\ 
\>Else\\
\>\>Inform the Manager of the store;\\
Else\\
\>Get the price of $i_1$ of quality $q_2$;\\
\>If (the price is less than $C_2$) \\
\>\>Buy $i_1$ of quantity $Q_2$;\\ 
\>Else\\
\>\>Inform the Manager of the store;\\
\rule{\linewidth}{0.02cm}
\end{tabbing}
\end{footnotesize}

\end{parawithheader}

\subsection{Notion of Mismatch between two Ontologies}
We focus on the following two types of mismatch between the client and server
ontologies in this paper.
\begin{description}
 \item[Specialization Mismatch(Type-1):] In this type of incompatibility the
client recognizes a class $c_2$ as the specialization of another class $c_1$
whereas the server recognizes $c_2$ as the specialization of some other class
$c'_1$. Our first example (Figure~\ref{Proto1}) is an instance of this type.

 \item[Attribute Assignment Mismatch(Type-2):] A very common type of
incompatibility arises where the client and the server both recognize classes
$c'_1, \ldots, c'_n$ as the specializations of another class $c_1$, but the
client associates an attribute $\alpha$ with the super class $c_1$, whereas the
server associates $\alpha$ with some of the sub classes $c'_i, \ldots, c'_j$, $0
< i,j \leq n$. Since we view the mismatches from the query answering
perspective, we use the notion of this conflict from the query perspective. If
the set of variables that is used in a query $q$, is not available at server
side, we denote that as \emph{attribute level(Type-2)} mismatch. Our first
example (Figure~\ref{Proto2}) is an instance of this type.
\end{description}

\section{Graph Model of Ontology}\label{GraphModelSec}
 While describing an ontology using OWL, the class and
the attributes(modeled as properties in the context of OWL) are used to
represent the meta-data. We use a graph based approach to model the meta-data
that are described as classes and attributes in OWL. While using OWL, the 
\emph{properties} are used to express the attributes. Therefore we use the term
property and attribute interchangeably. We define the ontology graph as follows.

\begin{definition}
A {\bfseries graph model} for an ontology $O$ is $\mathcal{G} = (V, E)$ where,
$V$ is the set of vertices and $E$ is the set of directed edges. Each node $v_i
\in V$ represents a class in the OWL ontology and $v_i$ is associated with a
{\bfseries property list} $\mathcal{L}(v_i)$ whose elements are the
data properties of the class.  The directed edges can be of the following types
\begin{description}
 \item[Inheritance-Edge :] An \emph{inheritance-edge} $e_{ij} \in E$ from $v_i$
to $v_j$, where $v_i, v_j \in V$, if $v_j$ is a sub class of
$v_i$.
 \item[Property-Edge :] An \emph{property-edge} $e_{ij} \in E$ from $v_i$ to
$v_j$, where $v_i, v_j \in V$, if $v_j$ is an object property of $v_i$.
\end{description}

\end{definition}
 
\section{Overview of the Method}\label{OntoMethodSec}
In this section we present the relevant formalisms and present the overall
algorithm for solving the problem. The \emph{variable set} and the \emph{class
set} specified in the query $q$ are denoted by $S_v(q)$ and $S_c(q)$
respectively. We present a graph search based structural matching algorithm to
check the semantic safety of the protocol.

\begin{definition}
The \emph{specialization sequence} $\sigma = \langle c_1. c_2. \;\cdots \;. c_k
\rangle$ in a query $q$ is the sequence of classes that are concatenated
through the `$.$' operator, and for any two consecutive classes $c_i$ and
$c_{i+1}$ in the sequence, $c_i$ is the super class of $c_{i+1}$. Therefore the
elements of $S_c(q)$ can be individual classes or specification sequences.
\end{definition}

\subsection{Structural Algorithm to Check the Semantic Consistency}

\begin{footnotesize}
\begin{algorithm}\label{MatchingAlgo}
\caption{Check-Consistency}
\SetCommentSty{sf}
\SetKwInOut{Input}{input}
\SetKwInOut{Output}{output}

\Input{The Protocol $\mathcal{P}$ and the Server Ontology $\mathcal{O}_s$}
\BlankLine
\Indp
$V \leftarrow \; \{\}$\;
\ForEach{query $q$ in the protocol $\mathcal{P}$}
{
  \ForEach{element $\tau$ in $S_c(q)$}
  {
    \eIf{$\tau$ is a specialization sequence}
    {
      $c_1 \leftarrow$ the first concept of $\tau$\;
      $c_t \leftarrow$ FindMatch($\mathcal{O}_s$, $c_1$)\;
      \For{$i \leftarrow 2$ \KwTo $length(\tau)$}
      {
        $c_m \leftarrow$ the $i^{th}$ concept of $\tau$\;
        \eIf{any class $c'_t$ equivalent to $c_m$ is not found as a sub
              class of $c_t$ in $\mathcal{O}_s$}
        {
          Report Mismatch at $c_m$\;
        }
        {
          $c_{t} \leftarrow \; c'_{t}$
        }
      }
      $V \leftarrow$ $V \; \cup$ property set for $c_{t}$\;
    }
    {
      \tcc{$c$ is an individual class}
      $c_1 \leftarrow \; \tau$\;
      $c_{t} \leftarrow$ FindMatch($\mathcal{O}_s$, $c_1$)\;
      $V \leftarrow$ $V \; \cup$ property set for $c_t$\;
    }
  }
  \If{$S_v(q) \subsetneq V$}
  {
    Report $\{S_v(q) - V\}$ as unmatched variables\;
  }
}
\Indm
\end{algorithm}

\begin{function}\label{FindMatchFunc}
\caption{FindMatch($\mathcal{O}_s$, $c_i$)}
\SetCommentSty{sf}

\Indp
Find the class $c_{t}$ which is equivalent to $c_i$ in $\mathcal{O}_s$\;
\If{$c_{t}$ is not found in $\mathcal{O}_s$}
{
        Report Mismatch at $c_i$\;
        exit\;
}
return $c_i$\;
\Indm
\end{function}
\end{footnotesize}

\subsection{Working Example}\label{workingExampleSec}
We present a working example to describe how the algorithm works. Consider the
protocol shown in Figure~\ref{Proto1}. We elaborate the steps of applying
Algorithm~\ref{MatchingAlgo} with respect to the fragments of the client
and server ontologies shown in Figure~\ref{OntoFragClient} and
Figure~\ref{OntoFragServ} respectively. These fragments are taken from the
benchmark provided by~\cite{OAEI}. The benchmark has one reference ontology and
four other real ontologies and the domain of these ontologies is bibliographic
references. We have used the reference ontology as the server ontology and
another real ontology named INRIA as the client ontology. We have used a
pictorial representation which is similar to entity-relationship diagram to
show the fragments of the ontologies. The classes are represented by the rounded
rectangles and the ovals represent the properties of a particular class. The
class hierarchy is shown using arrows, that is a sub class is connected to its
super class by an arrow which is directed towards the sub class. The properties
that belong to a particular class are connected to the rounded rectangle
corresponding to that class through a line.

\begin{description}
 \item[Step-1:] While applying Algorithm~\ref{MatchingAlgo} to the server
ontology, the individual class \emph{`Manual'} is searched and since the search
is successful, it is checked that the attributes that are associated with
class \emph{`Manual'} in the query in the protocol are actually answerable by
the server and this check turns out to be successful for the ontologies that
are presented here.

 \item[Step-2:] The next query uses the class \emph{`Book'}.
Algorithm~\ref{MatchingAlgo} performs the consistency checking in the way that
is similar to the previous query and the check is successful. 

 \item[Step-3:] The third query uses a specialization sequence
\emph{`Book.Proceedings'}. Algorithm~\ref{MatchingAlgo} searches for the
\emph{`Book'} class in the server ontology and then checks whether
\emph{`Proceedings'} is a sub class of \emph{`Book'} in the server ontology.
Algorithm~\ref{MatchingAlgo} reports a failure since in the server ontology
\emph{`Proceedings'} is not a sub class of \emph{`Book'}.
\end{description}

\begin{figure}[htb]
\centering
\includegraphics[scale=0.65]{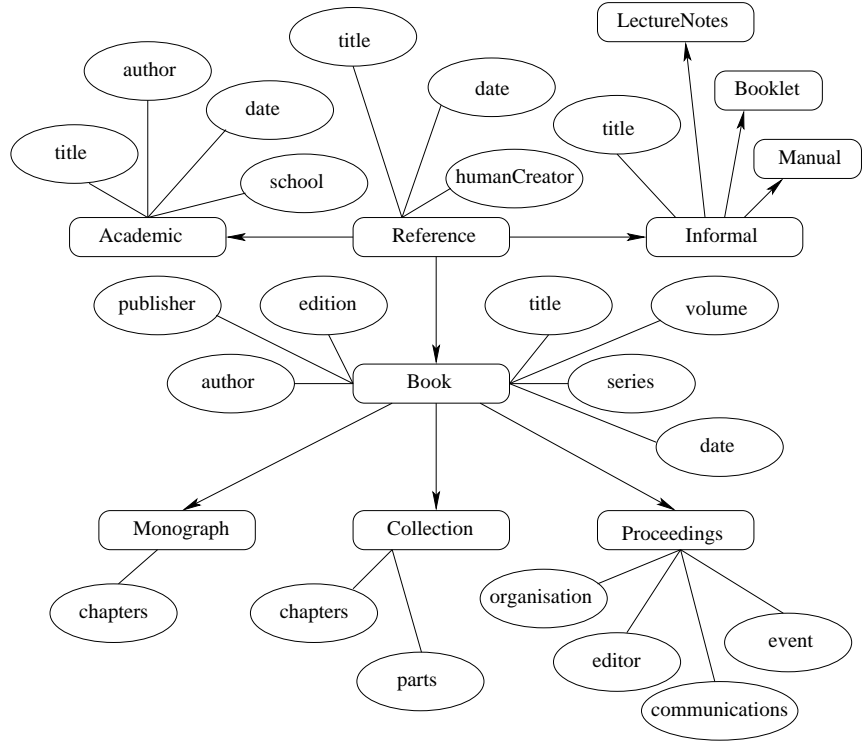}
\caption{Fragment of Client Ontology} \label{OntoFragClient}
\end{figure}

\begin{figure}[htb]
\centering
\includegraphics[scale=0.65]{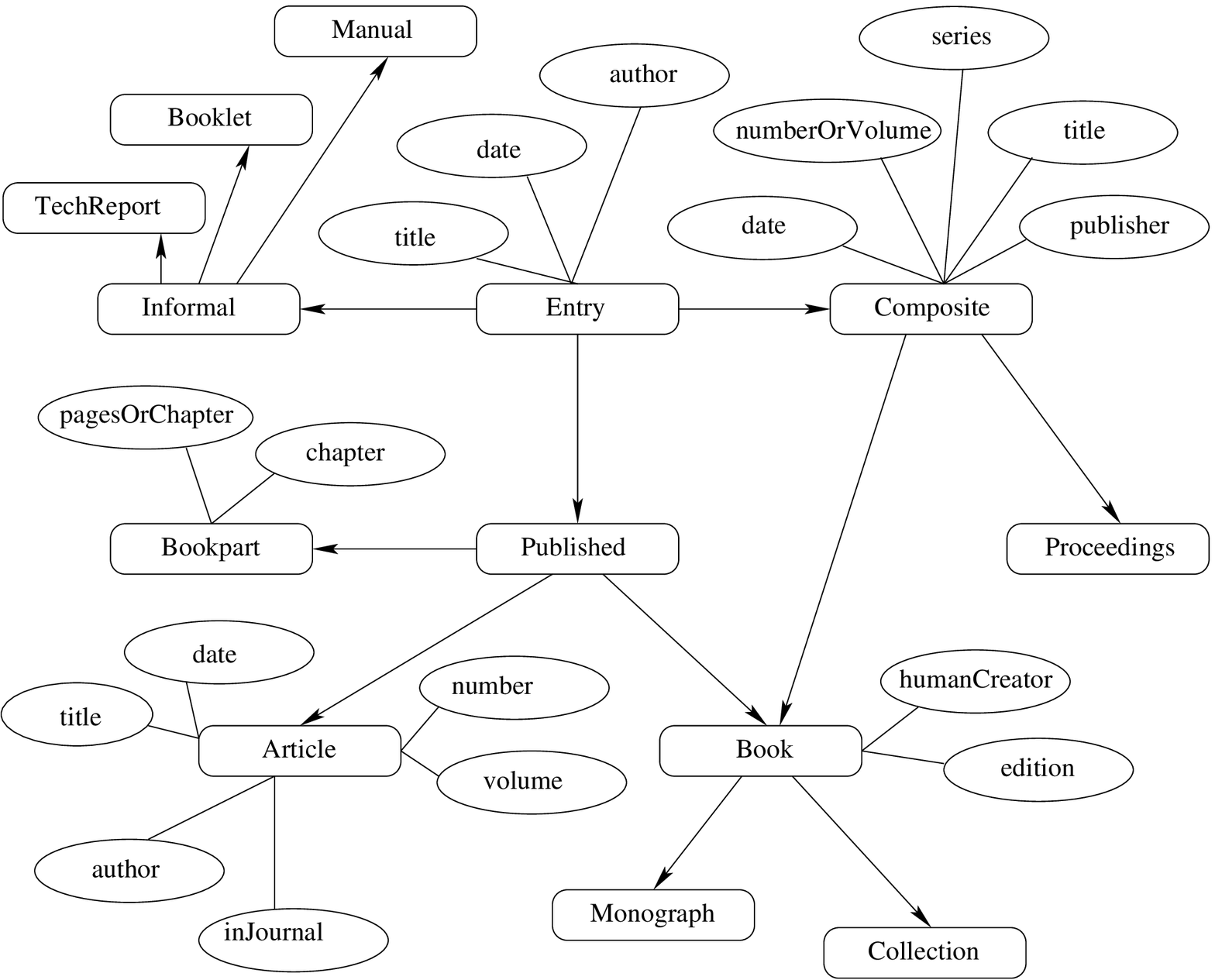}
\caption{Fragment of Server Ontology} \label{OntoFragServ}
\end{figure}

\subsection{Proof of Correctness}
\begin{theorem}{[Soundness]}\label{thm1}
The mismatches returned by Algorithm~\ref{MatchingAlgo} are correct.
\end{theorem}
\begin{proof}
Algorithm~\ref{MatchingAlgo} reports mismatch in three cases. We observe each
of the cases as follows.
\begin{description}
 \item[Mismatch in individual class:] If Algorithm~\ref{MatchingAlgo} does not
find a matching class $c$ which is used in a query, a conflict is reported.
Since the class is not recognized by the server, it is not possible for the
server to answer the query. Therefore the outcome of the algorithm is correct.

 \item[Mismatch in specialization sequence:] Consider a specialization sequence
$\sigma$ $=$ $\langle c_1. c_2. \;\cdots \;. c_k \rangle$ in a query $q$ on
which Algorithm~\ref{MatchingAlgo} returns a mismatch. We prove the correctness
of the consistency checking by induction on the length $k$ of $\sigma$. \\
{\itshape Basis($k=1$):} In this case there is only one class in the
specialization sequence and this case falls under the case of mismatch in
individual classes.\\
{\itshape Inductive Step:} Suppose Algorithm~\ref{MatchingAlgo} returns the
mismatch correctly for specialization sequences having length $k$. We prove
that Algorithm~\ref{MatchingAlgo} reports the conflicts correctly for the
specialization sequences having length $k+1$. There can be two possible
cases.
\begin{enumerate}
 \item The conflict is reported for a class that appears in the $i^{th}$
location of the sequence, where $1 < i < k+1$. The reported mismatch is correct
according to the inductive hypothesis.

 \item  The conflict is reported for the $k+1^{th}$ class of the sequence. In
this case there exists a matching specialization sequence at server ontology
up to length $k$. But $c_{k+1}$ is not a sub class of class $c_k$ according to
the server ontology. Therefore the conflict reported by
Algorithm~\ref{MatchingAlgo} is correct.
\end{enumerate}

 \item[Mismatch on variables:] Suppose the set of variables that are specified
by the client is $V_c$ in a query $q$ corresponding to the class set $S_c(q)$
and the failure is reported on some variable in $V_c$. Since
Algorithm~\ref{MatchingAlgo} first finds the matches corresponding to the
classes in $S_c(q)$ and then checks for the answerability with respect to the
variable set, in this case every class in $S_c(q)$ is matched with suitable
classes in the server side. Now Algorithm~\ref{MatchingAlgo} reports conflict if
there exists any variable that is not recognized by the server as an attribute
of at least one of the classes that correspond to the classes in $S_c(q)$.
Therefore the reported conflict falls under the \emph{Type-2 or attribute
level} conflict category. \qed
\end{description}

\end{proof}

\begin{theorem}{[Completeness]}\label{thm2}
For any protocol $\mathcal{P}$, if there is any mismatch of type-1 or type-2,
Algorithm~\ref{MatchingAlgo} reports it.
\end{theorem}
\begin{proof}
This proof is done by construction. For each of the type of the mismatches we
show that Algorithm~\ref{MatchingAlgo} uses a sequence of operations through
which the mismatch is detected. We present the proof for each mismatch type.
\begin{description}
 \item[Type-1 Mismatch:] Consider a specialization sequence $\sigma = \langle
c_1. c_2. \;\cdots \;. c_k \rangle$ which is used in query $q$.
Algorithm~\ref{MatchingAlgo} starts by finding the class that is equivalent to
$c_1$ at the server side. If there is only one class in $\sigma$ then
Algorithm~\ref{MatchingAlgo} reports mismatch when the corresponding class is
not found in the server ontology. When the length of $\sigma$ is greater than
$1$, Algorithm~\ref{MatchingAlgo} continues to check whether $c_i$ is a subclass
of $c_{i+1}$ where $1 < i < k$. A mismatch is reported by
Algorithm~\ref{MatchingAlgo} whenever $c_i$ is a subclass of $c_{i+1}$ for $1 <
i < k$. Hence if there exists any mismatch in any specialization sequence, the
algorithm reports it.
    
 \item[Type-2 Mismatch:] Consider a query $q$ made by the client and the set of
variables is $V_c$ in $q$. The set of classes is denoted by $S_c(q)$. We argue
that, if there exists a \emph{Type-2} mismatch for query $q$,
Algorithm~\ref{MatchingAlgo} reports it. For \emph{Type-2} mismatches
Algorithm~\ref{MatchingAlgo} first checks the presence of the equivalent classes
$c^s_i$ in the server ontology and computes the union $V_s$ of the attributes
corresponding to every $c^s_i$. If there is any variable/s in $V_c$ that are not
present in $V_s$, a conflict is reported by Algorithm~\ref{MatchingAlgo}. Hence
if there exists a \emph{Type-2} mismatch for a query,
Algorithm~\ref{MatchingAlgo} reports it. \qed
\end{description}

\end{proof}

\section{Ontology with Back-end Database}\label{OntoDatabaseSec}
In this section we describe the two level representation for describing
ontologies -- using OWL to describe the classification and using database
to store the instances. This type of representation is helpful for describing
domains with large number of instances. From the point of view of the instances
of classes, the classes in an ontology can be categorized as follows.
\begin{enumerate}
 \item  Classes of Abstract Type -- these classes are used for purely the
purpose of describing a domain in hierarchically. These classes does not have
any instances. They act only as the super class of other classes.
 \item  Classes with Instances -- these classes may act as super class of other
classes but they have a non-empty set of instances.
\end{enumerate}
Consider the ontology fragment in Fig.~\ref{OntoFragServ}. Here \emph{Entry},
\emph{Informal}, and \emph{Composite} are the example of abstract classes. On
the other hand, \emph{Book}, \emph{Monograph} etc. are the example of classes
with instances. Although \emph{Book} is a super class of \emph{Monograph} and
\emph{Collection}, it is possible to have instances of \emph{Book} which are
neither \emph{Monograph} nor \emph{Collection}. 

While using the two level representation, it is important to keep the database
schema consistent with the wrapper ontology. A choice of describing the database
schema could be maintaining a table for each of the non-abstract classes present
in the ontology. Alternative ways of describing the database are possible, but
we use this simplistic representation of the database schema to present the
proposed algorithm.

\subsection{Query Answering in the Presence of the Database}
When the server side adheres the two layer structure for its ontology, every
query in the protocol is answered by generating corresponding tuples from the
back-end database. In the context of the back-end database the occurrences of
variables in a protocol, can be categorized into the following types.

\begin{description}
 
 \item [Uninstantiated:] When a variable is placed in a query for the first
time without initialization, it is referred to as an \emph{uninstantiated}
occurrence of variable or in short \emph{uninstantiated variable}. The values
for the variables are instantiated at the side where the query is evaluated. 

 \item [Instantiated:] Other than the first occurrence without initialization,
all other occurrences of a variable is referred to as \emph{instantiated}
occurrence of that variable or in short \emph{instantiated variable}. At these
occurrences, the variables are already assigned to some value by the server.
These occurrences are used for value propagation.

\end{description}

\begin{parawithheader}{Evaluation Semantics of a Query}
The semantics of the evaluation of the query is similar to the \emph{Conjunctive
Datalog}. The evaluator of the query tries to assign value to uninstantiated
variables and forms a tuple which satisfies logical \emph{and} of the conditions
specified in the \emph{where} clause of the query. Same variables in different
classes specified in the \emph{where} clause of the query have to be assigned to
the same value.
\end{parawithheader}

\noindent
Consider the protocol presented in Fig.~\ref{Proto1}. In
Section~\ref{workingExampleSec} we have shown that the protocol has an
ontological conflict, when the client and the server uses the ontologies in
Fig.~\ref{OntoFragClient} and Fig.~\ref{OntoFragServ} respectively. Consider
the fact, that the condition, $(t_2 != null)$ may always evaluate false due to
the actual data that is stored in the database of the server. In that case, the
ontological conflict in the last query, $[\textsf{Get} (title:\; t3, author:\;
a, date:\;d2) \, \textsf{from} \, Book.Proceedings]$, will never be sensitized.
In other words the conflicts at the ontology level may turn out to be spurious.
We define the \emph{spuriousness} of an ontological conflict as follows.
\begin{definition}
An ontological conflict is \emph{spurious}, when for all possible correct
instantiations of the variables, the conflict is not reachable from the start
state of the protocol, due to the decisions taken at different stages of the
protocol. By \emph{correct} instantiations we mean the instantiations that
conform to the evaluation semantics defined earlier.
\end{definition}

\subsection{Related Formalisms}
Here we present the relevant formalisms for describing the algorithm to check
the presence of the conflict detected by Algo.~\ref{MatchingAlgo} at the current
state of the server database.

\begin{definition}
The {\em assignable set} of values for a variable $\varphi$ is the set of
values that can be assigned to $\varphi$ during the instantiation and it is
denoted as {\em AssignableSet($\varphi$)}.
\end{definition}
Suppose in a protocol ${\mathcal P}$, a query $q$ has variable set $v$ =
\{$\varphi_1, ..., \varphi_n$\} and concept set $C$ = \{$C_1, ..., C_m$\}. Let
us also assume that in ${\mathcal P}$ all the variables of $q$ are
\textit{uninstantiated} variables. The notion of assignable set in the presence
of the previously instantiated variables is discussed later. The evaluation of
the query basically assigns a values to each of the variables in that query. All
the variables together form a tuple $\tau$ = $\langle val_1, val_2, \ldots ,
val_n \rangle$ such that if any variable $\varphi_k$ is common between class
$C_i$ and class $C_j$ then both the classes have to assign same value to the
variable $\varphi_k$. All such possible tuples that can be populated by the
evaluator side, form the assignable set of values for $v$ and the assignable set
for a variable $\varphi_i$ is:
\[
AssignableSet(\varphi_i) = \{val \, | \, \exists \tau \in AssignableSet(v)
    \wedge \tau = \langle val_1, val_2, ..., val_n \rangle 
    \wedge val_i = val \}
\]
The \emph{dependencies} among the variables play an important role for
determining the AssignableSet for a variable.
\begin{definition}
In a query, if some of the variables are previously instantiated, we say that
the previously instantiated set of variables is \emph{constraining} the set of
values of the uninstantiated variables. 
Suppose in the same query $q$, among the variables specified in $q$, $\varphi_1,
\cdots , \varphi_k$ are previously instantiated and $\varphi_{k+1}, \cdots,
\varphi_n$ are the variables that are instantiated by the evaluation of $q$. We
define the {\em constrain relation} ${\mathcal R}_{\mathcal C}$ and the {\em
ConstrainSet} as follows.
\begin{align*}
{\mathcal R}_{\mathcal C} &= \{ (\varphi_i, \varphi_j)  \; \big{|} \;
     \text{ where } \varphi_i \in \{ \varphi_1, \cdots , \varphi_k \} 
     \text{ and } \varphi_j \in \{ \varphi_{k+1}, ..., \varphi_n \} \}\\
ConstrainSet(\varphi_i) &= \{ \varphi_{k+1}, \varphi_{k+2},  
     \cdots , \varphi_n
\} 
\end{align*}
\end{definition}
Consider the same query $q$. The AssignableSet for the set  of variables of $q$
is the set of all tuples $\tau$ = $\langle val_1,val_2, \ldots...,val_n \rangle$
such that the following conditions hold.

\begin{itemize}
 \item If any variable $\varphi_k$ is placed in more than one concepts, all the
concepts assign same values to $\varphi_k$.
 \item $(val_1 \in A_1) \wedge ... \wedge (val_k \in A_k)$, where $A_1,
\cdots, A_k$ are the assignable sets of variable $\varphi_1, \cdots, \varphi_k$
respectively.
\end{itemize}
\begin{definition}
The {\em RestrictSet} for a variable set  $v$ is obtained by computing the
transitive closure of the ${\mathcal R}_{\mathcal C}$ on $v$. 
\end{definition}
We use the notion of the \textit{split} operation on the assignable set of
values of a variable and it works as follows. Let a query, $q$, consists of
concept $C_i$ with a uninstantiated variable $\varphi_i$, and a previously
instantiated variable $\varphi_j$. Suppose a decision is made on the variable
$\varphi_j$. In each branch, the possible values of $\varphi_j$ forms a subset
of its assignable set. Since the value of $\varphi_i$ is dependent on
$\varphi_j$, in each branch the possible values for $\varphi_i$ also forms a
subset of the assignable set of $\varphi_i$.
\begin{definition}
The {\em SplitSet} for a variable set $v$ is a subset of RestrictSet($v$) 
and is defined as:
\begin{align*}
SplitSet(v) = \{\varphi_j \, | \, &\varphi_j \in RestrictSet(\varphi_i)
    \text{ and } \varphi_j \text{ appears in a condition in the path of the
    protocol}\\
  &\text{ from the start of the protocol to the query with ontological
    conflict }\varphi_i \}
\end{align*}
\end{definition}
\begin{definition}
\emph{RelevantConditionSet} of a variable set $v$ is the set of conditions in
true form on the variable set ${v}_{split}$, which have to be \textit{true} for
reaching the conflicting query.
\end{definition}


\subsection{Algorithm for Detecting Spurious Conflicts with respect to the
Back-end Databases}
\begin{footnotesize}
\begin{algorithm}[H]
\label{VerifyAllConflict}
\caption{Verify the Conflicts on Back-end Database}
\Indp
Initialize a hash table $H^t$ \;
\BlankLine
\tcc{\em In the hash table $H^t$, a set of variables $v$ forms the key,
which is mapped to the AssignableSet of the variable set $v$}
\BlankLine
\ForEach{conflicting query $q$}{
  $v \leftarrow$ The set of instantiated variables specified in $q$\;
  \uIf{VerifyConflict($v$)}
  {
    Report mismatch on variable $v$ at database level\;    
  }
  \lElse
  {
    Report the conflict as spurious\;
  }
}
\Indm
\end{algorithm}

\begin{function}
\label{VerifyConflictFunc}
\caption{VerifyConflict($v$)}
\Indp
${v}_{restrict} \leftarrow$ The RestrictSet for the variable set $v$\;
${v}_{split} \leftarrow$ The SplitSet for the variable set $v$\;
${v}^{s}_{restrict} \leftarrow$ MakeSets(${v}_{restrict}$)\;
Construct a priority queue ${\Gamma}$ of variable sets\;
\tcc{${\Gamma}$ is ordered according to the order of the instantiations of
     its variable sets}
\ForAll{variable set $v_i \in {v}^{s}_{restrict}$}
{
  Enqueue $v_i$ in ${\Gamma}$\;
}
Table set $S^t \leftarrow \{\}$\;
\While {$\Gamma$ is not empty}
{
  $u \leftarrow$ Dequeue ($\Gamma$)\;
  \eIf {(VerifyConflict($u$))} 
  {
    \tcc{The set of possible valuations for $u$ is not empty}
    $t \leftarrow$ Search $H^t$ and return the table containing $u$ \;
    \If {$t$ $\notin$ $S^t$}
    {
      $S^t \leftarrow S^t$ $\cup$ $\{t\}$\;
    }
  }
  {
    \tcc{The set of possible valuations for $u$ is empty, 
         so the conflict is spurious}
    return false;
  }
}
Find the query $q$ that instantiates variable set $v$\;\nllabel{getQuery}
\If {${v}_{split}$ != $\varnothing$}
{
  $c \leftarrow$ The RelevantConditionSet on the variable set ${v}_{split}$\;
  $\delta \leftarrow SplitAssignableSet(\delta, {v}_{split}, c)$\;
}
\eIf {$\delta == \varnothing$ }
{
  Report the conflict on $v$ as spurious\;
  return false;
}
{
  Insert $\delta$ in $H^t$ \;
  return true;
}
\Indm
\end{function}

\newpage
\begin{function}[H]
\label{MakeSetsFunc}
\caption{MakeSets($v$)}
\Indp
{
  Initialize set of variable sets $v^{ret} = \{\}$\;
  \While {$v$ is not empty}
  {
    Find a query $q$ that instantiates some of the variables in $v$\;
    Initialize variable set $v_{temp} = \{\}$\;
    \ForAll{variable $\varphi_i \in v$ \emph{and} $\varphi_i$ is
             instantiated by $q$}
    {
      $v \leftarrow v - \{\varphi_i\}$\;
      $v_{temp} \leftarrow v_{temp} \cup \{\varphi_i\}$\;
    }
    $v^{ret} \leftarrow v^{ret} \cup \{v_{temp}\}$\;
  }
}
\Indm
\end{function}

\begin{function}[H]
\label{GenerateAssignableSetFunc}
\caption{GenerateAssignableSet($q$, $S^t$)}
\Indp
{
  \tcc{Suppose $q$ is made with the concepts $C_1, ..., C_n$ and $\varphi_{i1},
       \ldots, \varphi_{ik}$ are the uninstantiated variables corresponding to
       the concept $C_i$}
  $v \leftarrow \{\varphi_{ij}$ $|$ $\varphi_{ij} \neq * \}$\;
  \eIf{$S^t == \Phi$}
  {
    \tcc{All the variables of $q$ are uninstantiated}
    Tuple set $T \leftarrow (C_1 \Join C_2 \Join ... \Join C_n)$\;\nllabel{T1}
  }
  {
    \tcc{Some of the variables of $q$ are previously instantiated and $t_1, ...,
         t_m$ $\in$ $S^t$ are the tuple sets corresponding to those variables}
    Tuple set $T \leftarrow (C_1 \Join C_2 \Join \ldots \Join C_n \Join t_1
                             \Join \ldots \Join t_m)$\;\nllabel{T2}
  }
  Relational algebra query $q^{Rel} \leftarrow \pi_{v}(T)$\;
  Compute $q^{Rel}$ and return the set of tuples\;
}
\Indm
\end{function}

\begin{function}[H]
\label{SplitAssignableSetFunc}
\caption{SplitAssignableSet($\delta$, ${v}_{split}$, $c$)}
\Indp
{ 
  \tcc{Suppose $c_1, \cdots , c_i \in c$}
  Relational algebra query $q^{Rel} \leftarrow \sigma_{(c_1 \vee c_2 \vee
      \ldots \vee c_i)}(\delta)$\;
  Compute $q^{Rel}$ and return the set of tuples\;
}
\Indm
\end{function}
\end{footnotesize}

\noindent
This algorithm can also be used by the server as the protocol
progresses(described as Scenario-4 in Section~\ref{IntroSec}). In that case, the
variables in the queries which are already executed, have some value assigned
to them and those variables will be considered as \emph{instantiated} by the
algorithm.

\subsection{Proof of Correctness}
The proof of correctness of Algo.~\ref{VerifyAllConflict} is presented below.
Algo.~\ref{VerifyAllConflict} verifies the spuriousness of conflicts returned by
Algo.~\ref{MatchingAlgo} on the server database. 

\begin{theorem}{[Soundness]}\label{thm3}
Algorithm~\ref{VerifyAllConflict} correctly reports the spuriousness of
conflict on the set of variables $v'$, where $v' = v \cup RestrictSet(v)$ and
$v$ is the set of previously instantiated variables in a query $q$ of protocol
$\mathcal{P}$ with ontological conflict. 
\end{theorem}
\begin{proof}
The proof is done using induction. We do the induction on the integer parameter
{\em n}, where {\em n} is the total number of {\em VerifyConflict} function
calls done by Algorithm~\ref{VerifyAllConflict} for $q$. Among the different
{\em VerifyConflict} function calls, first call is done by
Algorithm~\ref{VerifyAllConflict} and the others are recursive calls. 

\begin{parawithheader}{Basis (n = 1)}
In this case RestrictSet($v$) = $\phi$. In this case if the $AssignableSet(v)$
is $\varnothing$ Algo.~\ref{VerifyAllConflict} correctly reports the conflict as
spurious, otherwise Algo.~\ref{VerifyAllConflict} reports the conflict as not
spurious, which is correct. 
\end{parawithheader}

\begin{parawithheader}{Inductive Step} 
We assume that the spuriousness of a conflict reported for the queries with
ontological conflict in $n$ steps are true. We now prove that the spuriousness
of a conflict that is reported in $(n+1)$ steps are correct. Consider the {\em
VerifyConflict} function call at Algo.~\ref{VerifyAllConflict} and without loss
of generality, we can assume this function call as the $(n+1)^{th}$ function
call (in the order of returning of the function calls). Therefore the other
calls are recursive calls done by the {\em VerifyConflict} to itself. The
following two cases are possible.

\begin{enumerate}
 \item The conflict may be detected as spurious by some call which is not the
$(n+1)^{th}$ call. In this case the spuriousness of the conflict  is correct by
the inductive hypothesis.

 \item The conflict is detected as spurious at the $(n+1)^{th}$ call to {\em
VerifyConflict}. All other previous calls to {\em VerifyConflict} add a table
to $H^t$ and the set of tables are kept in $S^t$. After that, function
GenerateAssignableSet is called to compute the assignable set for the set of
previously instantiated variables $v$ in the query $q$ with ontological
conflict. It follows from the description of the function, that this function
restricts the set of valuations of $v$ by taking the natural join with the
valuations of variables in RestrictSet($v$). Since the conflict is not detected
as spurious in the variables in RestrictSet($v$), when the function detects the
conflict as spurious, the statement $\delta == \varnothing$ is true. Therefore
in the protocol $q$ is not reachable from the start state of the protocol. \qed

\end{enumerate}

\end{parawithheader}
 
\end{proof}

\begin{theorem}{[Completeness]}\label{thm4}
If there is a spurious conflict on the set of variables $v'$, where $v'$ = $v$
$\cup$ RestrictSet($v$) and $v$ is the previously instantiated variable set
specified in a query $q$ of protocol $\mathcal{P}$ with ontological conflict,
the algorithm reports it. We do the proof by establishing the contrapositive of
the statement, i.e. Algorithm~\ref{VerifyAllConflict} reports the ontological as
not spurious, if $q$ is reachable from the start state of $\mathcal{P}$.
\end{theorem}

\begin{proof}
Suppose $v'$ = $\{\varphi_1,...,\varphi_n\}$. Let the valuations of the
variables in $v'$ are $(val_1,...,val_n)$ when the conflict in $q$ is not
spurious. In this case the conflict may occur in the following way. Consider the
{\em VerifyConflict} function calls made to determine the spuriousness of the
ontological conflict in $q$, among which the first call is done by
Algo.~\ref{VerifyAllConflict} and the subsequent calls are recursive
calls. The conflict is detected as \emph{not} spurious, only if all the
recursive calls to {\em VerifyConflict} add a table to $H^t$ and the set of
tables are kept in $S^t$. Since the conflict is determined as \emph{not}
spurious, the statement $\delta$ is not empty. Therefore in $\mathcal{P}$, $q$
is reachable from the start state of the protocol using any instantiation of
variables belonging to $\delta$. \qed
\end{proof}

\section{Related Works}\label{RelatedWorksSec}
Different aspects of web service interaction have been an active area of
research. However most of these researches consider the interaction at syntactic
level. Foster {\em et. al.} addressed the compatibility verification of web
services in~\cite{DBLP:conf/icws/FosterUMK04}. They adopted a model based
approach for checking the compatibility of web services at different level of
abstraction. However the semantics of exchanged data is not addressed by the
researchers. In~\cite{DBLP:conf/www/Betin-CanBF05} researchers address the
interaction among web services which is asynchronous in nature and propose a
design pattern to help the development of composite web services based on
asynchronous interaction. Zhao {\em et. al.} provides a formal treatment of web
service choreography in~\cite{DBLP:conf/wsfm/ZhaoYQ06}. They define a formal
model of the of WS-CDL and propose a methodology to formally verify the
correctness of a choreography using the model checker SPIN. In
\cite{DBLP:conf/www/BeyerCH05} authors proposed a formalism for specifying the
web service interfaces. They discuss about three kind of constraints which can
be put by a web service interface. The {\em propositional constraints} are
imposed by an interface by specifying the methods that can be invoked by the
clients along with the constraints on the input and output parameters({\em
signature constraints}). \emph{Protocol Constraints} specify the temporal
requirements on the sequence of the method invocations. An algorithm is
proposed to check compatibility among the web services based on the mentioned
constraints. However all the proposed verification strategies work at a
syntactic level, without considering the semantics of the exchanged data. 

On the other hand the current research in semantic web is focused towards the
standardization of the ontology used by the web services with a vision of
computers becoming capable of analyzing all web data. Semantic matchmaking
\cite{DBLP:conf/dexaw/GuoLX05,DBLP:conf/IEEEcit/GuoCL05}
and discovery of semantic web services
\cite{DBLP:conf/widm/PathakKCH05,DBLP:conf/atal/KluschFS06,DBLP:conf/bpm/VuHA05}
are two important research directions in semantic web. The underlying objective
of these approaches is to 
compare facts belonging to different ontologies and to evaluate their
compatibility. Standards like RDF, OWL, WSML etc. are developed for this
purpose.

Ontology plays an important role towards enhancing the integration and
interoperability of the semantic web services. A significant amount of
research has been done towards formalizing the notion of conflict between
two ontologies. In~\cite{Visser_analysis}, authors present a detailed
classification of conflicts by distinguishing between 
\emph{conceptualization} and \emph{explication} mismatches. 
In~\cite{Reconcile3} authors further generalize the notion of conflicts and
classify semantic mismatches into language level mismatches and ontology
level mismatches. Then ontology level mismatches are further classified into 
conceptualization mismatch and explication mismatch. Further research in the
same direction \cite{DBLP:conf/grc/QadirFN07} adds few new types of 
conceptualization mismatches. Researchers in~\cite{DBLP:conf/er/LiL04}
present alternative types of conflicts that are primarily relevant to
OWL based ontologies. However primary focus of these works is towards the
interoperability between two ontologies -- rather than the correctness of the
protocol for information exchange with respect to the interpretation.

Ontology mapping primarily focuses on combining multiple heterogeneous
ontologies. In~\cite{Calvanese01aframework} authors address the problem 
of specifying a mapping between a global and a set of local ontologies.
In~\cite{Madhavan02representingand} authors discuss about establishing a 
mapping between local ontologies. In~\cite{Noy01anchor-prompt} the problem
of ontology alignment and automatic merging is addressed.

Significant amount of research has been done towards the development of the
protocol. In \cite{Paurobally03developingagent} researchers proposed a
methodology for developing protocols in a multi agent environment. They extend
propositional dynamic logic to formally specify the protocol and also use an
extension of state-charts for visual representation. In
\cite{DBLP:conf/prima/OluyomiS04} a step by step procedure is presented for the
development of web service interaction protocols from the problem definition to
the final specification. However these approaches are focused towards the
development of protocol for multi agent environment. The semantics of the
exchanged data is not addressed in these works.

The problem of checking compatibility between two ontologies
with respect to a protocol is new and to the best of our knowledge there is no
prior work on this topic. 

\section{Conclusion}\label{ConclusionSec}
In this paper we addressed the problem of detecting the presence of semantic
mismatch where the data exchange between two ontologies is defined in 
terms of a protocol. We believe that the proposed methodology will be very
helpful for the integration of web services that are developed independently.
Moreover the future of internet applications lie in exchanging knowledge, where
semantic conflict will be a major issue.

{
\footnotesize
\bibliographystyle{splncs}
\bibliography{ref}
\thispagestyle{fancy}
}
\end{document}